\documentclass[twoside,leqno,twocolumn]{article}  
\usepackage{ltexpprt} 

\usepackage{amsfonts}
\usepackage{subfig}
\usepackage{epsfig}
\usepackage{graphicx}
\usepackage{algorithmic}
\usepackage{amsmath,amssymb}
\usepackage{url}
\usepackage{algorithm}

\def\RR{{\mathbb R}}

\newcommand{\beq}{\begin{equation}}
\newcommand{\eeq}{\end{equation}}
\newcommand{\be}{\begin{enumerate}}
\newcommand{\ee}{\end{enumerate}}

\newcommand{\bi}{\begin{itemize}}
\newcommand{\ei}{\end{itemize}}
\newcommand{\bc}{\begin{center}}
\newcommand{\ec}{\end{center}}

\def\real{\hbox{\rm\setbox1=\hbox{I}\copy1\kern-.45\wd1 R}}

\begin{document}

\title{\Large Controlled Sparsity Kernel Learning\footnote{Work Done in the year 2011}}
\author{Dinesh Govindaraj \\ Bell Labs Research \\ Bangalore, India \\ dinesh.govindaraj@alcatel-lucent.com \\
\and
Raman Sankaran\\ Indian Institute of Science \\ Bangalore, India  \\ ramans@csa.iisc.ernet.in \\
\and
Sreedal Menon \\ Bell Labs Research \\ Bangalore, India \\ sreedal.menon@alcatel-lucent.com \\
\and
Chiranjib Bhattacharyya \\ Indian Institute of Science \\ Bangalore, India  \\ chiru@csa.iisc.ernet.in \\}
\date{}

\maketitle
 

\begin{abstract} \small\baselineskip=9pt 
Multiple Kernel Learning(MKL) on Support Vector Machines(SVMs) has been a popular front of research in recent times due to its success in application problems like Object Categorization.
This success is due to the fact that MKL has the ability to choose from a variety of feature kernels to identify the optimal kernel combination. 
But the initial formulation of MKL was only able to select the best of the features and misses out many other informative kernels presented. 
To overcome this, the $L_p$ norm based formulation was proposed by Kloft et. al. This formulation is capable of choosing a non-sparse set of kernels through a control parameter
$p$. Unfortunately, the parameter $p$ doesnot have a direct meaning to the number of kernels selected. We have observed that stricter control over the number of kernels
selected gives us an edge over these techniques in terms of accuracy of classification and also helps us to fine tune the algorithms to the time requirements at hand.
In this work, we propose a Controlled Sparsity Kernel Learning (CSKL) formulation that can strictly control the number of kernels which we wish to select. 
 The CSKL formulation introduces a parameter $t$ which directly corresponds to the number of kernels selected. It is important to note that a search in $t$ space is finite and
fast as compared to $p$. We have also provided an efficient Reduced Gradient Descent based algorithm to solve the CSKL formulation, which is proven to converge. Through our experiments 
on the Caltech101 Object Categorization dataset, we have also shown that one can acheive better accuracies than the previous formulations through the right choice of $t$.
\end{abstract}

\section{Introduction}
Support Vector Machines(SVMs) \cite{vapnik} have emerged as powerful tools for classification problems. The key to accurate classification using SVMs is the choice of 
\emph{Kernel functions}(for definition of Kernel function please see \cite{smola}). This issue was first studied in Lanckriet et. al. ~\cite{LaCrBaGhJo04} where the problem 
of Multiple kernel learning(MKL) was first introduced. They have been successfully applied to a variety of domains e.g. text, object recognition \cite{VaRa07, kumar07}, 
protein structures\cite{protein_mkl}. Even though the idea was to explore the space of all possible linear combinations of the specified kernels, 
the functional framework associated with it could only select the best kernel from the set of specified kernels.  
Recently, many other approaches have been proposed to overcome this limitation\cite{SaDiRa09, KlSo09}. While some of them select all the kernels and some have 
sparse solutions that choose a subset of the specified kernels in a weighted combination, none of them have explicit control over sparsity. 

Due to lack of explicit control, in many application scenarios Non-Sparse solutions end up selecting some bad kernels also which leads to reduction
in the discriminative power of the combination kernel. We show experimental evidence of this phenomenon. 
One might argue that if the kernels given were all good kernels, this problem will not persist. But that does not take away the fact that 
a lower-accuracy good kernel can still bring down the accuracy of a better kernel. In most of the recent publications, 
we do not get a glimpse of the original problem as the space of kernels explored is very small and most of the kernels have almost equal power of
representation.

While sparse solutions\cite{LaCrBaGhJo04, RaBaCaGr08} overcome this particular problem to an extent by having some inherent ability to select a 
combination of a subset of the specified kernels, once again, there is no way to control the sparsity of the solution. This inherits most of the 
problems of selecting one and selecting all kernels due to the lack of control. The most relevant problem is that it misses out on some important 
features by selecting lesser number of kernels than optimal. In the case of applications like Object Recognition, 
the necessity of Non-Sparse solutions have been brought to light \cite{eth_biwi_00649,SaDiRa09}. This is due to the fact that different 
kernels represent different features necessary for the task and dropping some of them or most of them will lead to a bad combination kernel. 
These flaws are shown in our experiments as well. 

This work builds a variable sparsity solution that has explicit control over the number of kernels selected overcoming all these problems. 
We show the effect and need of strict control of sparsity through our experiments on the application of Object Recognition. 
Along the way, we have also extended the MKL framework to nu-SVMs which allow us better control of the number of support vectors and training error as well.

In the following section \ref{sec:rel}, we present a review of the existing work on MKL. 
Section \ref{sec:bmkl} introduces the {\tt CSKL} formulation for the {\tt C-SVM} and {\tt $\nu$-SVM}. 
Section \ref{sec:algo} presents the algorithms to solve the proposed formulations. 
Section \ref{sec:expts} demonstrates the usefulness of the {\tt CSKL} formulation on a toy dataset and the Caltech101 real world Object Categorization dataset.

\section{Related Work}\label{sec:rel}
Multiple Kernel Learning(MKL) was initially proposed by Lanckriet et. al. ~\cite{LaCrBaGhJo04}. 
They introduced an Semi-Definite programming(SDP) approach to solve for the combination kernel. 
As SDP becomes intractable with increase in size and number of kernels, Bach et.al \cite{BaLaJo04} reformulated 
MKL by considering each feature as a block and applying the $l_1$ norm across the blocks and $l_2$ norm within each block.
For this formulation several algorithms\cite{SoRaScBe06,RaBaCaGr08,SaDiRa09} were proposed to speed up the optimization process. 
\cite{SoRaScBe06} provides an Semi-Infinite Linear Programming(SLIP) based algorithm which decreases the training time to large extent.
SimpleMKL \cite{RaBaCaGr08} proposed by Rakotomamonjy et.al. derived a formulation which is equivalent to the block $l_1$ norm based formulation and provided
 a Reduced Gradient Descent based algorithm that is faster than the SLIP algorithm proposed previously. 
The dual of the SimpleMKL formulation is given by,
\begin{eqnarray}\label{eqn:smkl_dual}
 \max_{\alpha_i}& \sum_i \alpha_i - \frac{1}{2} \sum_{i,j} \alpha_i \alpha_j y_i y_j \left( \sum_m \gamma_m K_m\left(x_i,x_j\right) \right) \nonumber \\
 \textit{s. t.} & \sum_i \alpha_i y_i = 0 \nonumber \\
& 0\le \alpha_i \le C \forall i \nonumber \\
& \sum_{m} \gamma_i = 1, \ \ \gamma_i \ge 0, \forall i 
\end{eqnarray}

While all these approaches discussed a sparse solution to MKL, on understanding the need for non-sparse solutions, researchers have been exploring the space of non-sparse formulations
in recent times.
To acheive non-sparsity, \cite{SaDiRa09} group the kernels and apply $l_{\infty}$ norm across the groups and $l_1$ norm within the groups. 
They have also proposed a Mirror Descent Algorithm for solving MKL formulations which is much faster than SimpleMKL. Especially when number of kernels are high.
Kloft et.al.\cite{KlSo09} apply general $l_p$ norm to kernels and they show that Non-Sparse MKL generalizes much better than sparse MKL. 
The dual of the $l_p$ norm based MKL formulation as proposed by Kloft et.al. looks like
\begin{eqnarray}\label{eqn:kloft_dual_1}
 \max_{\alpha_i}& \sum_i \alpha_i - \frac{1}{2} \left( \sum_m \left( \sum_{i,j} \alpha_i \alpha_j y_i y_j K_m\left(x_i,x_j\right) \right)^\frac{p}{p-1} \right)^\frac{p-1}{p} \nonumber \\
& \textit{s. t.} \sum_i \alpha_i y_i = 0 \nonumber \\
& 0\le \alpha_i \le C \forall i
\end{eqnarray}
They have shown that when $p=1$, the formulation is equivalent to SimpleMKL 
and as $p$ moves to $\infty$, it explores non-sparse solutions. But the value of $p$ lacks a direct meaning or implication to the number of kernels selected.


Even though the details of sparse and non-sparse solutions have been explored, none of these formulations have explicit control of sparsity for their solutions.
As we have demonstrated in the experiments section, strict control of sparsity is highly valuable. Hence we propose a formulation, 
where we can parametrically control the total number of kernels selected and an efficient reduced gradient descent based algorithm to solve it. 
We have also experimentally shown that our formulation will be able to better state-of-the-art 
performance on the Caltech101\cite{caltech101} dataset for object categorization through strict control of sparsity.

\section{Controlled Sparsity Kernel Learning}\label{sec:bmkl}
In this section we introduce the new Controlled Sparsity Kernel Learning (CSKL) formulation and prove that this formulation can explicitly control the sparsity 
of kernel selection through a parameter $t$. We derive the CSKL formulation by modifying the dual of MKL ~\cite{LaCrBaGhJo04}. Lets start with the MKL dual~\cite{LaCrBaGhJo04}

\begin{align}
\min_{\gamma \ge 0,K}  & \omega(K) (= \max_{\alpha \in S_m}- \frac{1}{2}\alpha^\top YKY\alpha + \sum_{i=0}^m \alpha_i) \nonumber \\
 & tr(K) = \delta \nonumber \\
& K = \sum_{i=1}^n \gamma_i K_i 
\end{align} 
where $S_m = \{ \alpha \in R^m | 0 \le \alpha \le C, y^T C=0 \}$.\\
Denote by $d_j\frac{t_j}{\delta} = \alpha^\top YK_jY \alpha $ and $t_j = Trace(K_j)$.
As $\omega(K)$ is convex, one can interchange the min and the max. 
Now, the dual looks like
\begin{align}\label{eqn:mkl}
 \max_{\alpha \in S_n,d} \min_{\gamma \ge 0} - \sum_j \gamma_jd_j\frac{t_j}{\delta} + \sum_i \alpha_i \nonumber \\
 \gamma^\top t = \delta \nonumber \\
d_j = \alpha^\top YK_jY \alpha 
\end{align} 
Define $\gamma'_j = \gamma_j \frac{t_j}{\delta}$ then the constraint 
$\gamma^\top t = \delta$ can be rewritten as $\sum_j \gamma_j' = 1$. The $l_\infty$ norm can be represented as
\begin{equation}
 \|v\|_{\inf} = max_{\sum_i \gamma_i \le 1, \gamma_i \ge 0} \sum_j \gamma_jv_j
\end{equation}
 for any $v_j \ge 0$ . Given this, Equation (\ref{eqn:mkl}) can be restated as 
\begin{align} \label{eqn:dual_mkl}
 \max_{\alpha \in S_n,d}  - \|d\|_{\infty} + \sum_i \alpha_i \nonumber\\
d_j = \alpha^\top YK_jY \alpha 
\end{align} 

This formulation (\ref{eqn:dual_mkl}) results in a sparse selection of kernels as shown in \cite{KlSo09}.
Similarly, equation (\ref{eqn:kloft_dual_1}) can also be rewritten as,
\begin{align} \label{eqn:kloft_dual}
 \max_{\alpha \in S_n,d}  - \|d\|_{p^*} + \sum_i \alpha_i \nonumber\\
d_j = \alpha^\top YK_jY \alpha \\
p^* = \frac{p}{p-1}
\end{align} 
The above formulation(\ref{eqn:kloft_dual}) is referred to as {\tt $L_{p}$ MKL} throughout this paper. 
Even though above formulation (\ref{eqn:kloft_dual}) uses generic norm over $d$, there is no guarantee of explicit control over sparsity.
In next section, we derive our CSKL formulation by modifying the norm on $d$.

\subsection{CSKL formulation}
Let $ v \in \RR^n_+$ denote the space of $n$ dimensional vectors with all components positive, i.e. $v_i > 0$. Let $v_{(i)}$ be the $i$th largest component of $v$, i.e. 
$v_{(1)} \ge v_{(2)} \ldots, v_{(n)}$
Consider the following convex function on $g_t: \RR^n_+ \rightarrow \RR$, 
 $g_t(v) = \sum_{i=1}^t v_{(i)}$ where $t$ is a positive integer less than $n$. 

We present our first claim by this theorem  
\begin{theorem} \label{theorem:B_norm}
If $v \in \RR_+^n$ such that $v_{(n)} > 0$ and $g_t(v)$ defined as before then  
\begin{align} \label{eqn:dual_budget}
 g_t(v) =& \max_{\gamma}  \  \gamma^T v \nonumber \\
s.t. & \sum_{i=1}^n \gamma_i = t, \ \ 0 \le \gamma_i \le 1, \forall i  & 
\end{align}
and at optimality $\gamma_i = 1, \mbox{whenever~} v_i > v_{(t)}$ and $\gamma_i = 0, \mbox{whenever~}  v_i < v_{(t)} $
\end{theorem}
\begin{proof}
We begin by constructing the Lagrangian of the problem 
\begin{align} 
L(\gamma,a,\mu,\beta) = & \gamma^{\top} v - a \left( \sum_{i=0}^n \gamma_i - t \right) \\ 
 & + \sum_{i=1}^n \beta_i \gamma_i - \sum_{i=0}^n  \mu_i \left( \gamma_i - 1 \right)
\end{align}
where the lagrange multipliers are $\mu,\beta$ and $a$.
Apart from the feasibility conditions on $\gamma$ and the non-negativity constraints on the lagrange 
multipliers $\mu$ and $\beta$ the KKT conditions reads as 
\begin{align}
\frac{\partial L}{\partial \gamma_i} = 0 \implies a +  \mu_i = \beta_i + v_i \label{eqn:b_g}\\
\alpha \left( \sum_{i=0}^n  \gamma_i - t \right)  = 0 \\
\beta_i \gamma_i = 0 \\
\mu_i \left( \gamma_i - 1 \right) = 0
\end{align}
The proof hinges on that fact that 
 $a = v_{(t)} $ satisfies the KKT conditions.
We note that both $\beta_i$ and $\mu_i$ cannot be simultaneously positive.
If $v_i < a $, then \eqref{eqn:b_g} could be obtained by setting $\beta_i > 0$ and 
$\mu_i =0$. As $\beta_i > 0$ then $\gamma_i = 0$. 
Again if $v_i > a $, then \eqref{eqn:b_g} could be obtained by setting $\beta_i = 0$ and  $\mu_i  > 0$. As $\mu_i > 0$ then $\gamma_i = 1$. 
Interestingly note that
if  $v_i = a $ as both $\mu_i= \beta_i =0$ and $1 > \gamma_i > 0$. 
Let us now suppose that $a = v_{(t)}$
The constraint 
$ \sum_{i=1}^n \gamma_i = t $
can now be written as 
$$ \underbrace{\sum_{i: v_i < v_{(t)}} \gamma_i}_{T1} + \underbrace{\sum_{i: v_i = v_{(t)}} \gamma_i}_{T2} + \underbrace{\sum_{i: v_i> v_{(t)}} \gamma_i}_{T3}=t $$
Due to observations made before it is straightforward to see that  
 $T1 = 0$ and  $T3 \le t -1$. One can always choose feasible  
 $\gamma_i \ \ \forall\  v_i = v_{(t)}$ such that $T_2 = t - T_3$  
This establishes the fact that $a = v_{(t)}$ indeed satisfies the KKT conditions
and for which $\gamma_i = 1 (\gamma_i =0) $ if $v_i > v_{(t)}(v_i < v_{(t)})$.
    
As KKT conditions are necessary and sufficient for this problem \cite{boyd}
we see that at optimality   
$\gamma^\top v = \sum_{i=1}^t v_{(i)}$ 
obtained by substituting the $\gamma$ obtained before.
This completes the proof.
\end{proof}
By introducing $g_t(d)$ to the dual (As in Eqn. \ref{eqn:dual_mkl}) we get the following CSKL formulation,
\begin{align} \label{eqn:bmkl_svm}
 \max_{\alpha \in S_n,d}   - g_t(d) + \sum_i \alpha_i \nonumber\\
d_j = \alpha^\top YK_jY \alpha \nonumber \\
\end{align} 
Note that CSKL formulation (\ref{eqn:bmkl_svm}) explicitly controls the sparsity of kernel selection by varying $t$ as is evident from Theorem \ref{theorem:B_norm}.
\subsubsection{$\nu$-CSKL}
A variant of SVM is the $\nu$-SVM \cite{SCHOLKOPF} where parameter $C$ is replaced by a parameter $\nu=[0,1]$. Here, the parameter $\nu$ is lower bound on the 
fraction of number support vector and an upper bound on the fraction of margin errors. In this section we extend our CSKL formulation to $\nu$-SVM. 
The dual of $\nu$-SVM is given by,
\begin{align}
\max_{\alpha} & \ \ - \frac{1}{2}\alpha^\top YKY\alpha\nonumber \\
\textup{s.t.} & \ \  0\le\alpha_i\le\frac{1}{m}, \sum_{i=1}^{m}\alpha_iy_i=0, \sum_{i=1}^{m}\alpha_i \ge \nu \nonumber \\
\end{align} 
Introducing MKL to the dual of $\nu$-SVM and rewriting it similar to equation (\ref{eqn:dual_mkl}).
\begin{align} \label{eqn:nusvm_dual}
\max_{\alpha} & \ \ - \|d\|_{\infty} \nonumber \\
\textup{s.t.} & \ \  0\le\alpha_i\le\frac{1}{m}, \sum_{i=1}^{m}\alpha_iy_i=0, \sum_{i=1}^{m}\alpha_i \ge \nu \nonumber \\
& \ \ d_j = \alpha^\top YK_jY\alpha 
\end{align} 
We now introduce our CSKL formulation in the setting of $\nu$-SVM.
\begin{align} \label{eqn:nubmkl}
\max_\alpha & \ \ -g_t(d) \nonumber \\
\textup{s.t.} & \ \  0\le\alpha_i\le\frac{1}{m}, \sum_{i=1}^{m}\alpha_iy_i=0, \nu\ge\sum_{i=1}^{m}\alpha_i \nonumber \\
& \ \ d_j = \alpha^\top YK_jY \alpha 
\end{align}
The above formulation is denoted as {\tt $\nu$-CSKL} throughout this paper. 
\section{Algorithms for solving CSKL formulations} \label{sec:algo}
We present an alternating optimization scheme for solving the $\nu-CSKL$ formulation. For a fixed $\gamma$, we solve the following maximization for $\alpha$, 

\begin{align} \label{eqn:nubmkl_expanded}
\max_\alpha & \ \ -\gamma^T d \nonumber \\
\textup{s.t.} & \ \  0\le\alpha_i\le\frac{1}{m}, \sum_{i=1}^{m}\alpha_iy_i=0, \nu\ge\sum_{i=1}^{m}\alpha_i \nonumber \\
& \ \ d_j = \alpha^\top YK_jY \alpha 
\end{align}

Note that in above problem $\gamma$ should satisfy the conditions $\ \ \sum_{i=1}^n \gamma_i = t, \ \ 0 \le \gamma_i \le 1, \forall i  $. We can use standard Sequential Minimal Optimization(SMO) solver for the above problem. Once optimal $\alpha^*$ is calculated, we compute $d$ as,
$d_j = \alpha^{*\top }YK_jY \alpha^*$. Next step is to solve for $\gamma.$ We can find the optimal $\gamma$ by solving $g_t(d)$ using Reduced Gradient Descent or a Linear Programming based Gradient Descent.

\begin{algorithm}                      
\caption{$\nu$-CSKL Algorithm}          
\label{alg1}                         
\begin{algorithmic}                    
\REQUIRE $ x^TK_jx > 0, \forall x \neq 0 , \forall j$
\STATE INPUT: $N=$ number of kernels
\STATE $\gamma = \frac{t}{N}$
\STATE $objOld = 0$
\WHILE{$\delta \le \epsilon$}
\STATE Solve $\alpha$ using SMO solver with kernel $ K = \sum_j^N \gamma_j  K_{j}$
\STATE Compute $ d_j = \alpha^T Y K_{j} Y \alpha$
\STATE Solve $g_t(d)$ using Reduced Gradient Descent or Linear Programming based solver
\STATE $obj = -\frac{1}{2}\alpha^T Y \left(\sum_{j=1}^N\gamma_jK_{j}\right) Y \alpha$ 
\STATE $\delta = obj - objOld$
\STATE $ objOld = obj $
\ENDWHILE
\end{algorithmic}
\end{algorithm}

\begin{algorithm}
 \caption{Reduced Gradient Algorithm for Solving $g_t(d)$}
 \label{algo:redgrad}
\begin{algorithmic}
  \STATE $J(\alpha, \gamma) = -\frac{1}{2}\alpha^T Y \left(\sum_{j=1}^N\gamma_j K_{j} \right) Y \alpha $ 
  \STATE Set $\mu=argmin_{m}\left(abs\left(\gamma_m-0.5\right)\right)$
  \STATE Set $J_{new}=J-1$, $\gamma_{new}=\gamma$
  \STATE Compute $\phi_m = \frac{\delta J}{\delta \gamma_m}$ for $m=1,\dots,N$
  \STATE Compute the descend direction $D(d,\mu,\phi)$
  \STATE Set $D_{new}=D$
  \WHILE{$J_{new} < J$}
    \STATE $\gamma=\gamma_{new}$
    \STATE $D=D_{new}$
    \STATE $\nu_1 = \underset{\left(m|D_m<0\right)}{argmin} -\frac{\gamma_m}{D_m}$
    \STATE $\nu_2 = \underset{\left(m|D_m>0\right)}{argmin} \frac{1-\gamma_m}{D_m}$
    \STATE $S_{max}=min\left(-\frac{\gamma_{\nu1}}{D_{\nu1}},\frac{1-\gamma_{\nu2}}{D_{\nu2}}\right)$ 
    \STATE $\nu =\underset{\nu_1,\nu_2}{argmin}\left(-\frac{\gamma_{\nu1}}{D_{\nu_1}},\frac{1-\gamma_{\nu2}}{D_{\nu_2}}\right)$
    \STATE $\gamma_{new}=\gamma +S_{max}D$, $D_{new}\left(\mu\right)=D_{\mu}-D_{\nu}$
    \STATE $D_{new}\left(\nu\right)=0$
    \STATE Compute $J_{new}(\alpha, \gamma_{new})$
  \ENDWHILE
  \STATE Linesearch along $D$ for $S \in \left[ 0,S_{max}\right]$
  \STATE $\gamma \leftarrow \gamma + S D$
 \end{algorithmic}
\end{algorithm}

In Algorithm \ref{algo:redgrad}, we present our Reduced Gradient Algorithm to solve $\gamma$.
The SVM solver is used to obtain $J$ (see Algorithm \ref{algo:redgrad}). The Descent Direction $D$ is defined as per Algorithm \ref{algo:descentdir}. In the case of $\nu$-CSKL, the value of $J\left(\alpha, \gamma\right) = -\gamma^Td$ while in the case of C-CSKL it is $-\gamma^Td+\alpha^Te$ while the rest of the framework remains the same.

\begin{algorithm}
 \caption{Calculating the Descent Direction}
 \label{algo:descentdir}
\begin{algorithmic}
\STATE INPUT: Kernel Weights $\gamma$, Selected Pivot $\mu$, Calculated Gradients $\phi$
\STATE OUTPUT: Optimal Descent Direction $D(\gamma,\mu,\phi)$
\FOR{$m=1$ to $N$}
\STATE $ D_m = 0 $
\IF {$\left(\gamma_m==0 \right.$ \& $\left. \phi_m - \phi_\mu > 0 \right)$}
\STATE $ D_m = 0 $ 
\ELSIF {$\left(\gamma_m==1\right.$ \& $\left.\phi_m - \phi_\mu < 0 \right)$}
\STATE $ D_m = 0 $ 
\ELSIF {$\left(\gamma_m>0\right.$ \& $\left.m!=\mu\right)$}
\STATE $ D_m =  -\phi_m - \phi_\mu$ 
\ELSIF {$\left(m==\mu\right)$}
\STATE $ D_m = \sum_{\nu!=\mu,\gamma_\nu>0}\left(\phi_\nu - \phi_\mu\right)$ 
\ENDIF
\ENDFOR
\end{algorithmic}
\end{algorithm}

Due to our assumptions on $K$, in both the cases, $J$ is convex and differentiable with Lipschitz gradient wrt. $\gamma$ \cite{Bonnans}.
For such functions the Reduced Gradient Method converges with bounds as defined in \cite{Luenberger84linearand}.

We also present a linear programming based approach to solve for $g_t(d)$. We use some standard LP Solver to solve the following linear program for finding descent direction $D$ for $g_t(d)$.
\begin{align}\label{eqn:linearprog}
\min_D & \ \ {\phi}^T D \nonumber \\
& D^T1 = 0, {}-\gamma \ge D\ge 1-\gamma
\end{align}
where $\phi_m=\frac{\partial J}{\partial \gamma_m}$ and step size $(S)$ can be found by using line search.
$\gamma$ is updates as $\gamma_{new}=\gamma+SD$.
Though this algorithm is found to converge for $K>0$ we have no bounds on its convergence as yet.

\section{Experiments.} \label{sec:expts}
To illustrate the benefits of CSKL formulation, we give results on Synthetic data and the Caltech101 \cite{caltech101} real-world Object Categorization dataset. We compare {\tt CSKL} algorithm with {\tt SimpleMKL} [Equation \ref{eqn:smkl_dual}] \footnote{Implementation downloaded from http://asi.insa-rouen.fr/enseignants/~arakotom/code/mklindex.html} which is a sparse selection algorithm, and {\tt $L_p$ MKL} [Equation \ref{eqn:kloft_dual_1}] with $p=2$ which is a non-sparse selection algorithm \footnote{Implementation available in the Shogun toolbox : http://www.shogun-toolbox.org/}. 
\subsection{Datasets}
In this section, we describe the datasets we used for our experiments.
\subsubsection{Synthetic Dataset}
To show the effect of noisy kernels, we generated $n=18$ kernels out of which $16$ are informative kernels and $2$ are noisy kernels. To build these kernels,
we sampled $m=500$ datapoints with dimension $d=3$ from two independent Gaussian distributions with covariance as the identity matrix 
and different means($\mu_1=0$ and $\mu_2=3$, Datapoints sampled from different Gaussians are assumed to belong to different classes).
We generated four kernels (two gaussian($\sigma=$ and $\sigma=$) and two polynomial kernels($\sigma=$ and $\sigma=$)) for each dimension seprately and all together($4*3+4=16$).
On top of this, we also added two carefully chosen noisy kernels to this kernel set.

\subsubsection{Caltech101}
The Caltech101 dataset has 102 categories of images such as airplanes, cars, leopards, etc. It has been shown by \cite{VaRa07,kumar07,eth_biwi_00649} that 
multiple image descriptors aid in the generalization ability of the learnt classifier. Using the method followed 
by \cite{VaRa07} \footnote{http://www.robots.ox.ac.uk/~vgg/software/MKL/v1.0/}, we extract the following 4 descriptors 
: PhowColor, PhowGray, GeometricBlur and SelfSimilarity. Each descriptor gives rise to a distance matrix. 
We create multiple Gaussian kernels for each descriptor by varying the Gaussian width parameter used to generate the kernel. 
We currently used 5 width values in the log space of -4 to 0. Hence we arrive at a total of 20 kernels. 
The number of binary classification problems are 5151 and 102, for 1-vs-1 and 1-vs-rest classification approaches respectively. 

\subsection{Need for Control Over Sparsity}
The key result we wish to establish is that by suitable variation of parameter $t$ in {\tt CSKL}, 
one can combine good kernels and eliminate noisy kernels and achieve better generalization than other MKL formulations. \\

In the Synthetic dataset setting $t = 1$ will facilitate sparse selection, 
and $t = n$ facilitates a complete non-sparse selection. As shown in the figure \ref{toy}, 
the {\tt CSKL} formulation clearly outperforms both sparse and non-sparse MKL by setting $t=4$. 
It is clear that setting $t=4$ in {\tt CSKL} gives better generalization performance than both $t=1$ and $t=18$. 
This clearly shows neither sparse nor complete non-sparse is good for this dataset. {\tt CSKL} 
is the only formulation which can capture all good kernels but still eliminate the noisy kernels by tuning parameter $t$. \\

In order to demonstrate that neither sparse nor non-sparse solutions are always the best in real-world datasets, 
we take all the binary 1-vs-1 and 1-vs-Rest classifiers in Caltech101 dataset and compare {\tt SimpleMKL} and {\tt $L_2$ MKL} solutions.
Figures \ref{fig:kloft_vs_SimpleMKL}, \ref{fig:kloft_vs_SimpleMKL_1vrest} show the ratio of improvement in accuracy of {\tt $L_2$ MKL} over {\tt SimpleMKL}.
 It is evident from the figures that neither of the algorithms are always the best. 
Thus, depending on the binary classification problem, we need to have different controls on the sparsity to achieve the state-of-the-art performance. 
Clearly this motivates that, to achieve the desired sparsity, we can use the {\tt CSKL} formulation instead of either {\tt SimpleMKL} or {\tt $L_2$ MKL}. 
In next section we show how {\tt CSKL} can achieve better performance than other algorithms.

\begin{figure}
\centering
\subfloat[Toy Dataset]{\includegraphics[width=3in,height=3in]{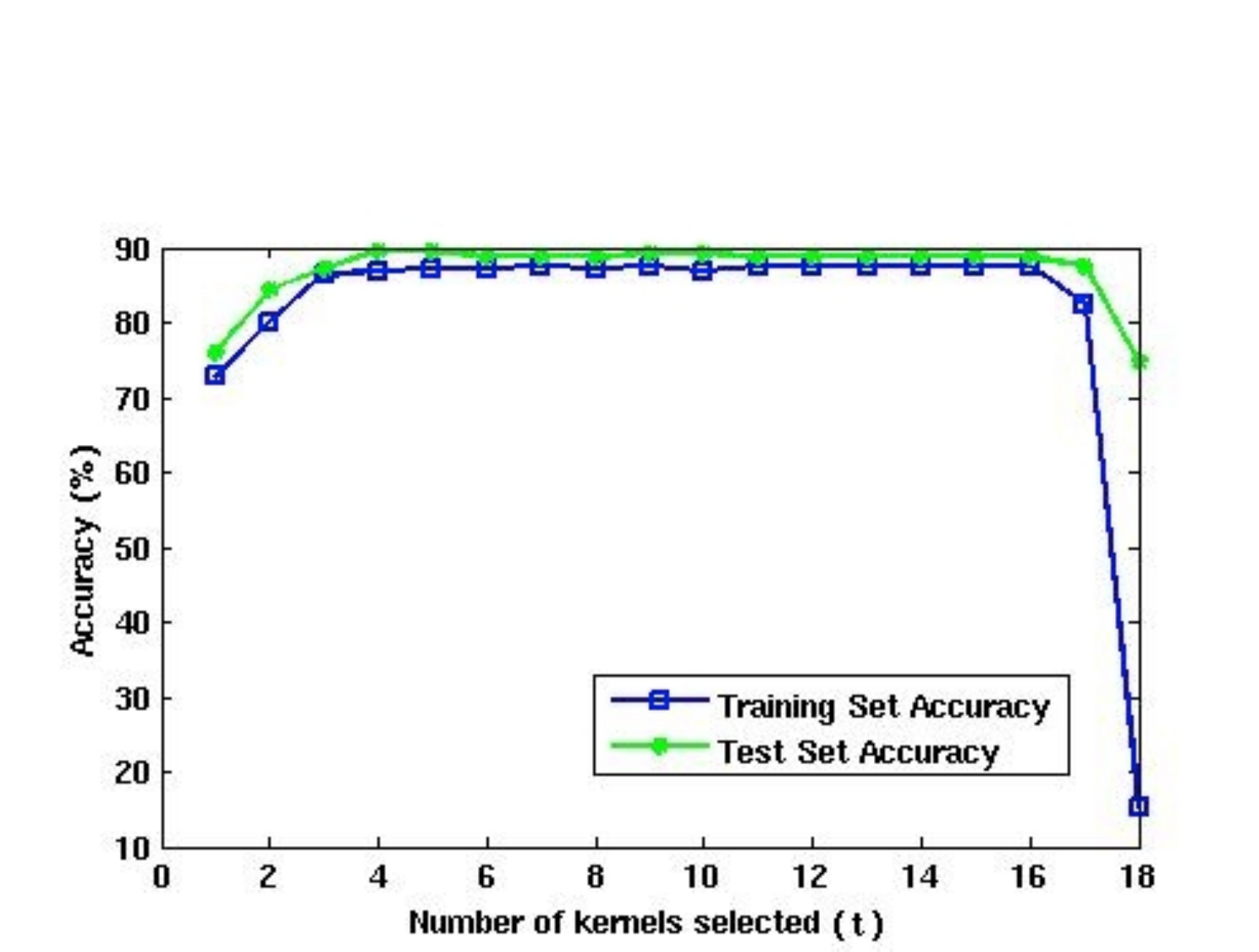}}
\caption{Plot of average accuracy achieved with \textbf{$\nu$-CSKL} on the Toy Dataset with respect to the parameter $t$.} \label{toy}
\end{figure}

\begin{figure}
\centering
\subfloat[Caltech-101]{\includegraphics[width=3in,height=3in]{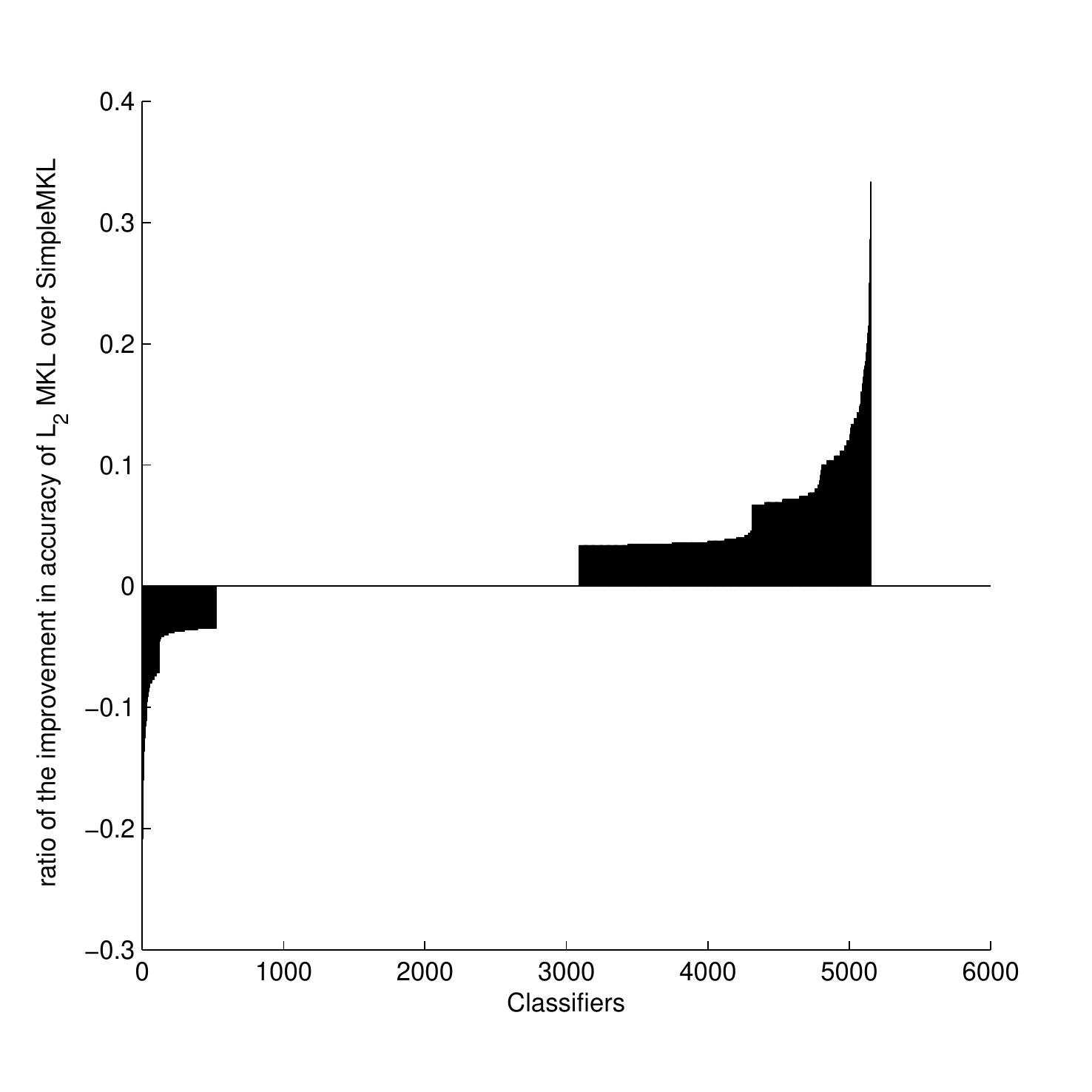}}
\caption{Ratio of improvement in accuracy of {\tt $L_2$ MKL} over {\tt SimpleMKL} across all binary 1-vs-1 classifiers in Caltech101 dataset } \label{fig:kloft_vs_SimpleMKL}
\end{figure}
\begin{figure}
\centering
\subfloat[Caltech-101]{\includegraphics[width=3in,height=3in]{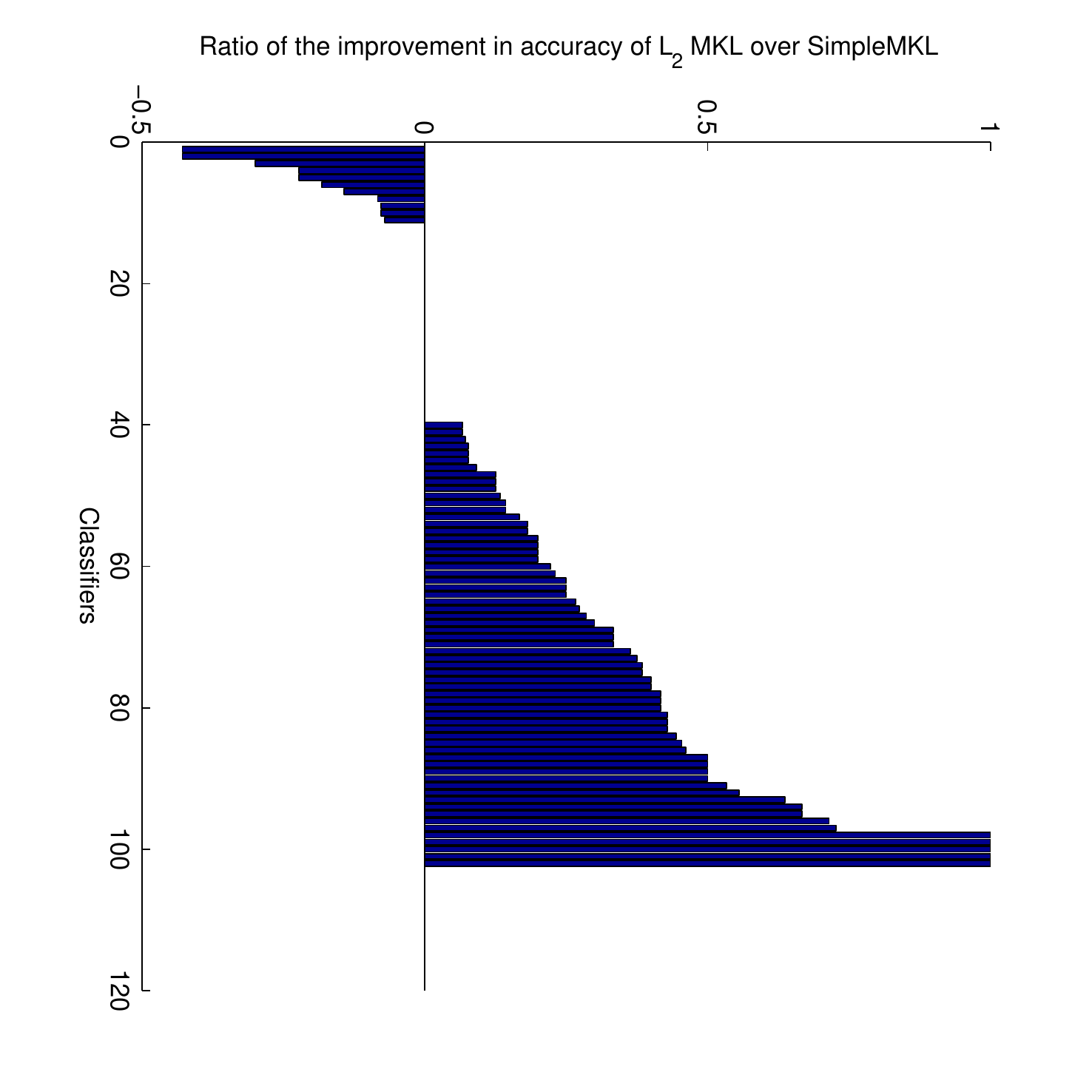}}
\caption{Ratio of improvement in accuracy of {\tt $L_2$ MKL} over SimpleMKL across all binary 1-vs-Rest classifiers in Caltech101 dataset} \label{fig:kloft_vs_SimpleMKL_1vrest}
\end{figure}

\subsection{Performance of CSKL}
We apply our {\tt CSKL} algorithm, and compare its performance against the other state-of-the-art algorithms {\tt SimpleMKL}  and {\tt $L_{2}$ MKL}. 
We take the highest accuracy achieved by {\tt CSKL} across various values of parameter $t$ for the comparison. 
Figures \ref{fig:BMKL_vs_Kloft_1v1}, \ref{fig:BMKL_vs_SimpleMKL_1v1} show the ratio of improvement in accuracy of {\tt CSKL} 
over {\tt $L_2$ MKL} and {\tt SimpleMKL}. 

\begin{figure}
\centering
\subfloat[Caltech-101]{\includegraphics[width=3.5in,height=3.5in]{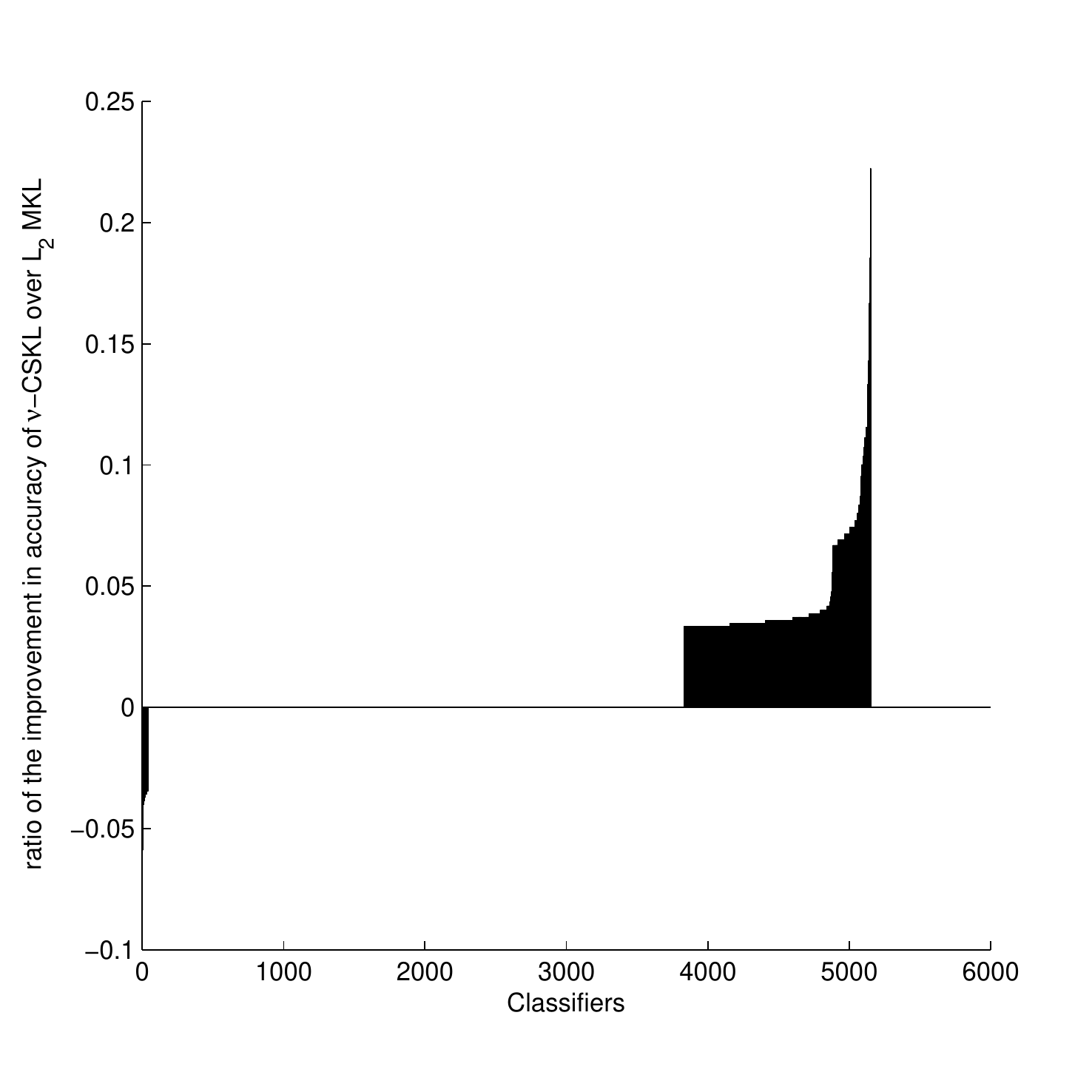}}
\caption{Ratio of improvement in accuracy of {\tt $\nu$-CSKL} over {\tt $L_2$ MKL} across all binary 1-vs-1 classifiers in Caltech101 dataset } \label{fig:BMKL_vs_Kloft_1v1}
\end{figure}
\begin{figure}
\centering
\subfloat[Caltech-101]{\includegraphics[width=3.5in,height=3.5in]{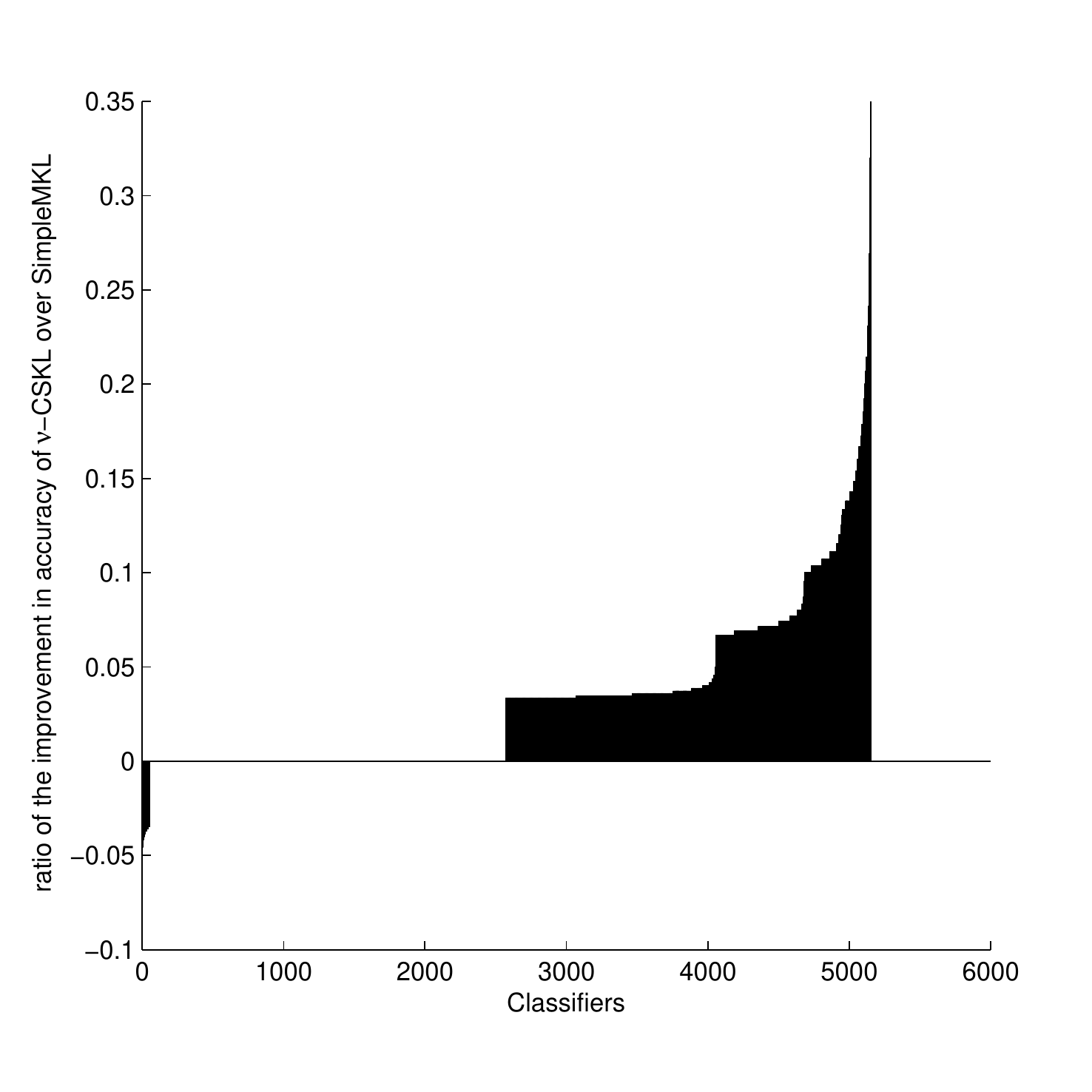}}
\caption{Ratio of improvement in accuracy of {\tt $\nu$-CSKL} over {\tt SimpleMKL} across all binary 1-vs-1 classifiers in Caltech101 dataset} \label{fig:BMKL_vs_SimpleMKL_1v1}
\end{figure}

\begin{figure}
\centering
\subfloat[Caltech-101]{\includegraphics[width=3.5in,height=3.5in]{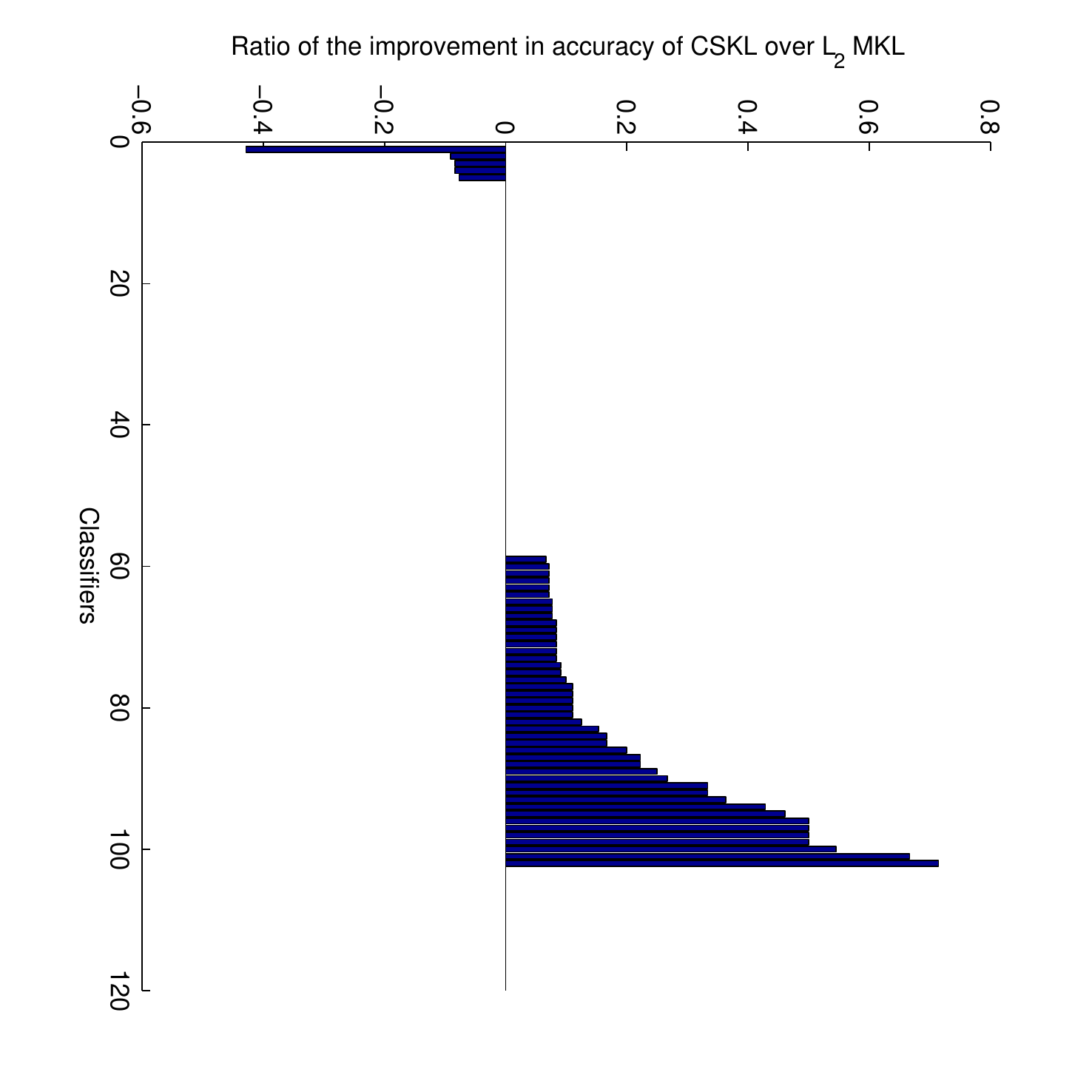}}
\caption{Ratio of improvement in accuracy of {\tt CSKL} over {\tt $L_2$ MKL} across all binary 1-vs-Rest classifiers in Caltech101 dataset } \label{fig:BMKL_vs_Kloft_1vrest}
\end{figure}
\begin{figure}
\centering
\subfloat[Caltech-101]{\includegraphics[width=3.5in,height=3.5in]{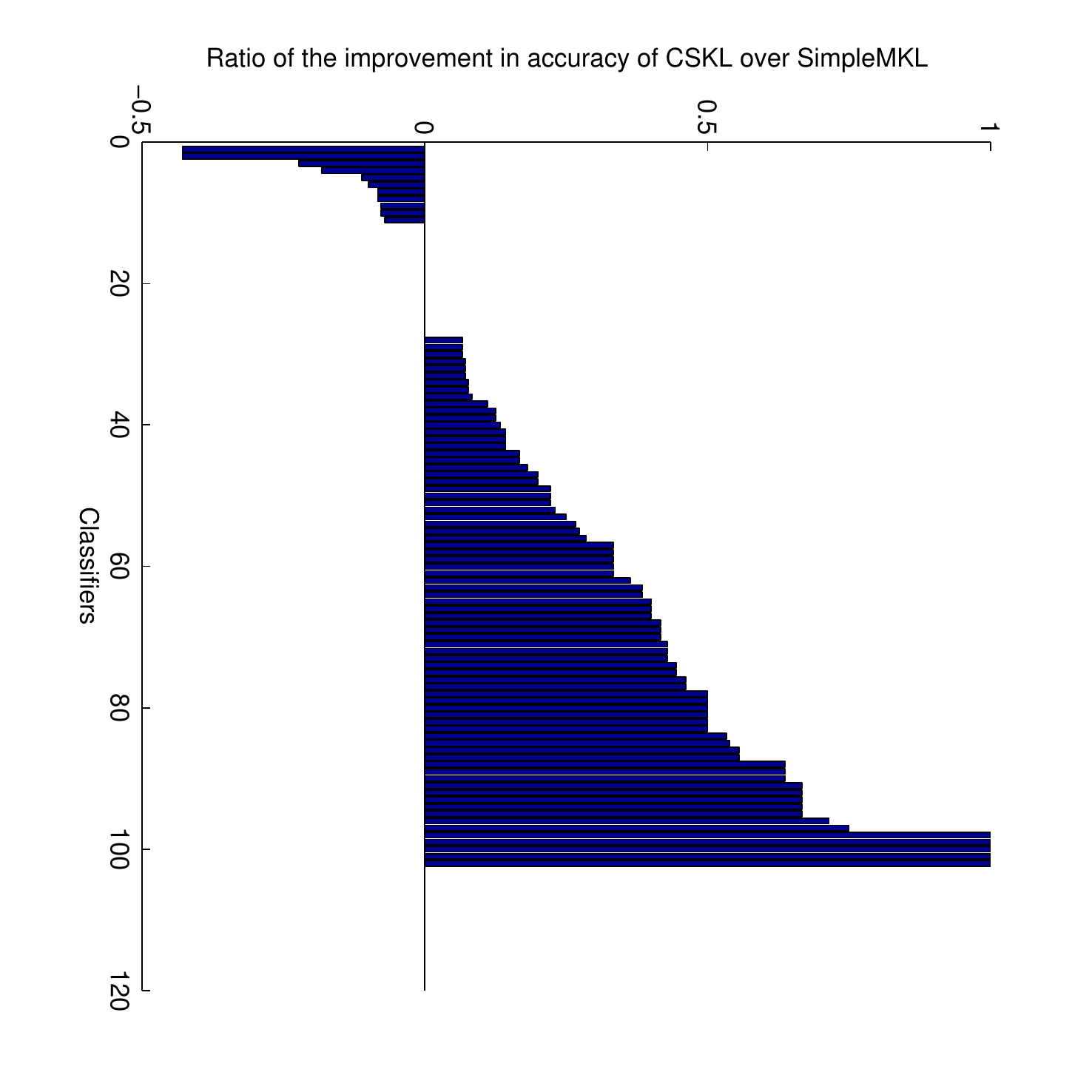}}
\caption{Ratio of improvement in accuracy of {\tt CSKL} over {\tt SimpleMKL} across all binary 1-vs-Rest classifiers in Caltech101 dataset} \label{fig:BMKL_vs_SimpleMKL_1vrest}
\end{figure}

We also present here overall performance of {\tt CSKL} on Caltech101 dataset. Figure \ref{fig:caltech101} shows the performance of {\tt CSKL} as $t$ is varied. For comparison, we have shown a straight line which shows the average accuracy achieved by {\tt SimpleMKL} and {\tt $L_{2}$ MKL}.
Figure \ref{fig:caltech101} clearly shows that all the 20 kernels are not necessary, since the {\tt CSKL} accuracy more or less saturates after $t > 4$. 
The result also shows that a sparse selection algorithm like {\tt SimpleMKL} wont be most efficient algorithm in terms of the accuracy achieved. 
And the performance of {\tt CSKL} is almost equal to that of {\tt $L_{2}$ MKL}, but the latter selects all the provided kernels, while we can achieve competitive accuracy with the former itself at $t = 4$. Note that no other formulation can give this flexibility to users to select exactly four best performing kernels. It is natural to use use $t=4$ here because number of descriptors used is four. Hence the experiments demonstrated in this section provide a proper justification for the usage of the {\tt CSKL} formulation.

\begin{figure}
\centering
\subfloat[Caltech-101]{\includegraphics[width=3in,height=3in]{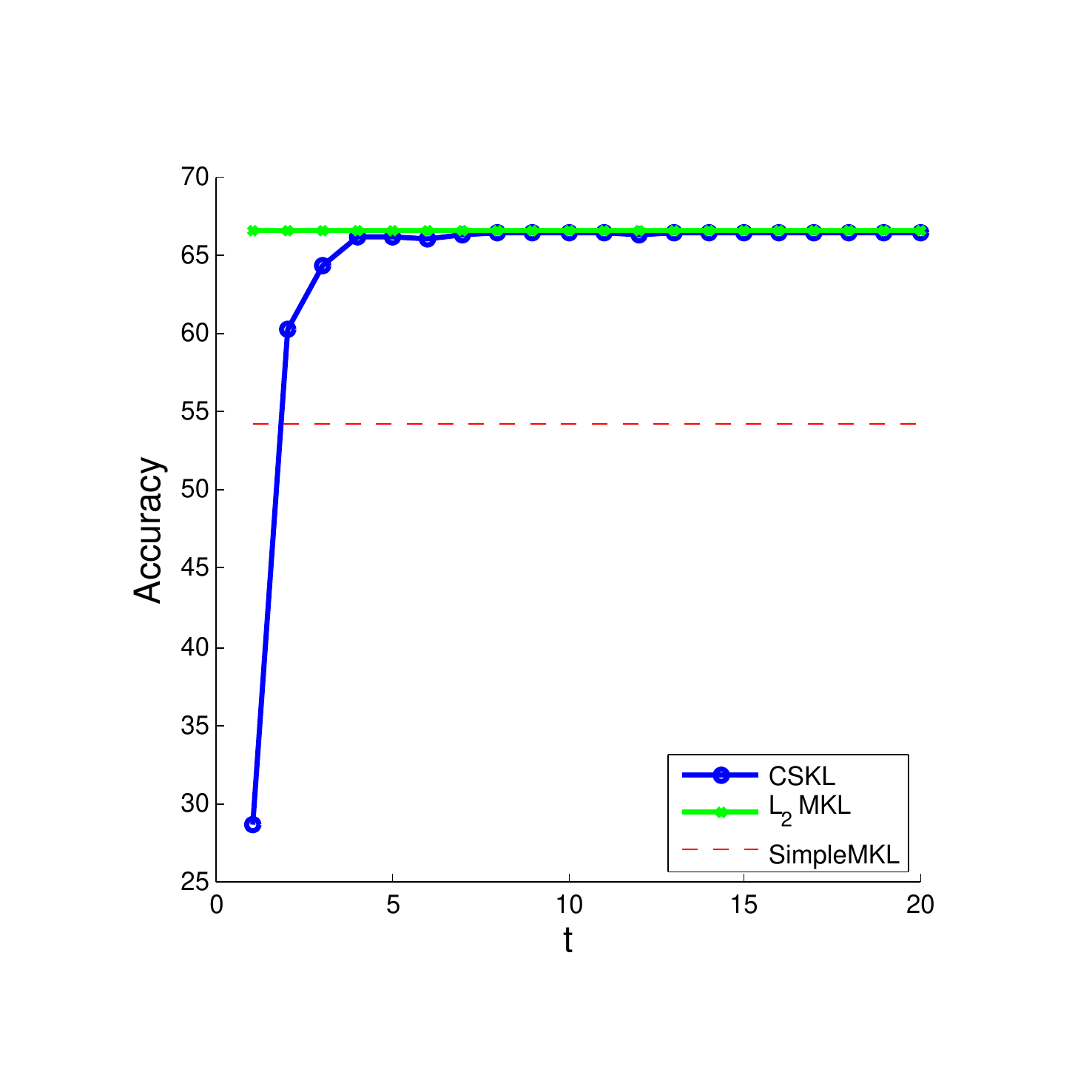}}
\caption{Plot of average accuracy with {\tt CSKL} on the Caltech101 dataset with respect to the parameter $t$.} \label{fig:caltech101}
\end{figure}

\subsection{Discussion}
To analyze more on why {\tt CSKL} achives better accuracy, we plot the histogram of number of descriptors selected when {\tt CSKL} outperforms {\tt $L_2$ MKL} and {\tt SimpleMKL} in the binary classification problems. We see that, in the figures \ref{fig:SimpleMKL_vs_BMKL_1v1_hist_bar}, \ref{fig:SimpleMKL_vs_BMKL_1vrest_hist_bar}, {\tt SimpleMKL} only selects one or two descriptors whereas {\tt CSKL} select all the descriptors.  For the cases where {\tt SimpleMKL} doesnt perform the best, non-sparse combination might be a better choice, and this is emperically confirmed in the figures \ref{fig:SimpleMKL_vs_BMKL_1v1_hist_bar}, \ref{fig:SimpleMKL_vs_BMKL_1vrest_hist_bar}.
Similarly from the figures \ref{fig:Kloft_vs_BMKL_1v1_hist_bar}, \ref{fig:Kloft_vs_BMKL_1vrest_hist_bar}, we see that {\tt $L_2$ MKL} selects all the descriptors whereas {\tt CSKL} does not select all descriptors most of the cases. These are expected, since, for the cases where {\tt $L_2$ MKL} perform low, it may mean that a non-sparse classification is preferable. And the same is reflected in the figures \ref{fig:Kloft_vs_BMKL_1v1_hist_bar}, \ref{fig:Kloft_vs_BMKL_1vrest_hist_bar}.

Out of the 5151 binary classification problems in the 1-vs-1 setting, {\tt $\nu$-CSKL} performed better in 5112 and 1498 cases against {\tt $L_2$ MKL} and {\tt SimpleMKL} respectively. Similarly in the 1-vs-Rest setting, out of the 102 classification problems, the numbers turned out to be 
97 and 91 against {\tt $L_2$ MKL} and {\tt SimpleMKL} respectively.

Finally, Figure \ref{fig:caltech101_desc} shows the number of descriptors selected as $t$ is increased. 
We can infer that as $t$ increased beyond 4, all the descriptors are selected. This is also not surprising since all the 4 descriptors used in our experiment are independent and experimentally they have been shown to aid the accuracy of the Object Categorization problem.


\begin{figure}
\centering
\subfloat[Caltech-101]{\includegraphics[width=3in,height=3in]{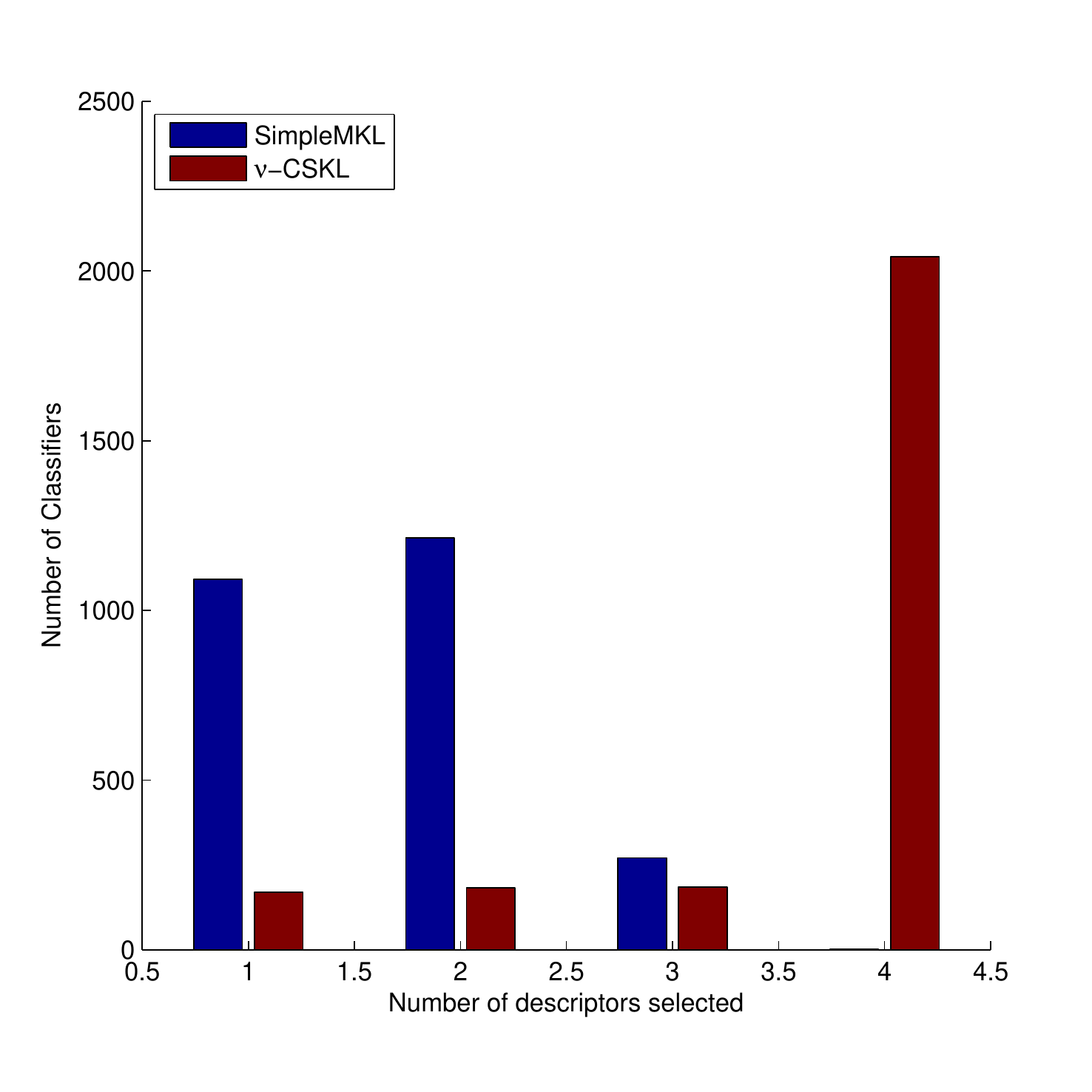}}
\caption{Histogram of number of descriptors selected where {\tt $\nu$-CSKL} performs better than {\tt SimpleMKL} in Caltech101 1-vs-1 binary classification problems} \label{fig:SimpleMKL_vs_BMKL_1v1_hist_bar}
\end{figure}
\begin{figure}
\centering
\subfloat[Caltech-101]{\includegraphics[width=3in,height=3in]{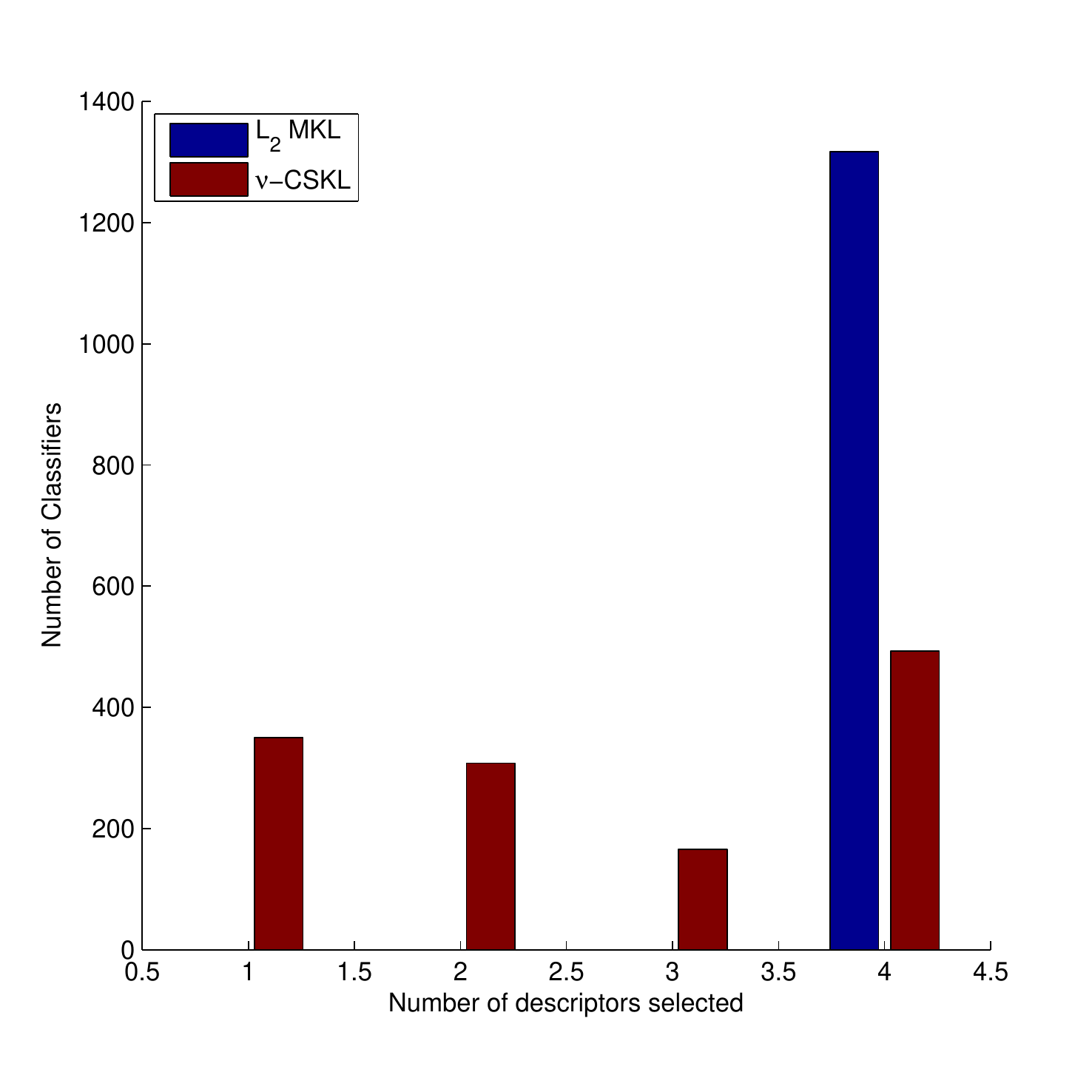}}
\caption{Histogram of number of descriptors selected where {\tt $\nu$-CSKL} performs better than {\tt $L_2$ MKL} in Caltech101 1-vs-1 binary classification problems} \label{fig:Kloft_vs_BMKL_1v1_hist_bar}
\end{figure}

\begin{figure}
\centering
\subfloat[Caltech-101]{\includegraphics[width=3in,height=3in]{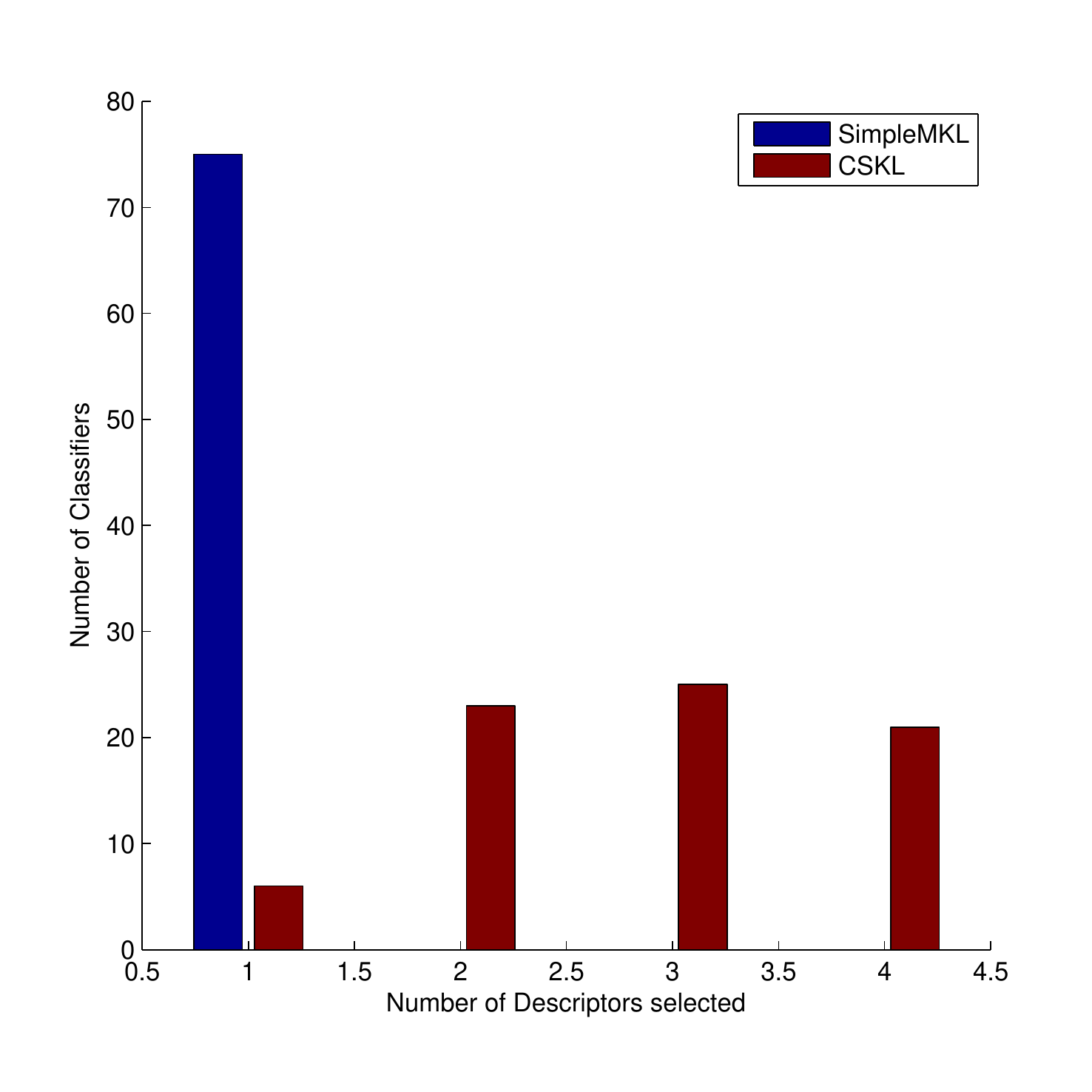}} 
\caption{Histogram of number of descriptors selected where {\tt CSKL} performs better than {\tt SimpleMKL} in Caltech101 1-vs-Rest binary classification problems} \label{fig:SimpleMKL_vs_BMKL_1vrest_hist_bar}
\end{figure}
\begin{figure}
\centering
\subfloat[Caltech-101]{\includegraphics[width=3in,height=3in]{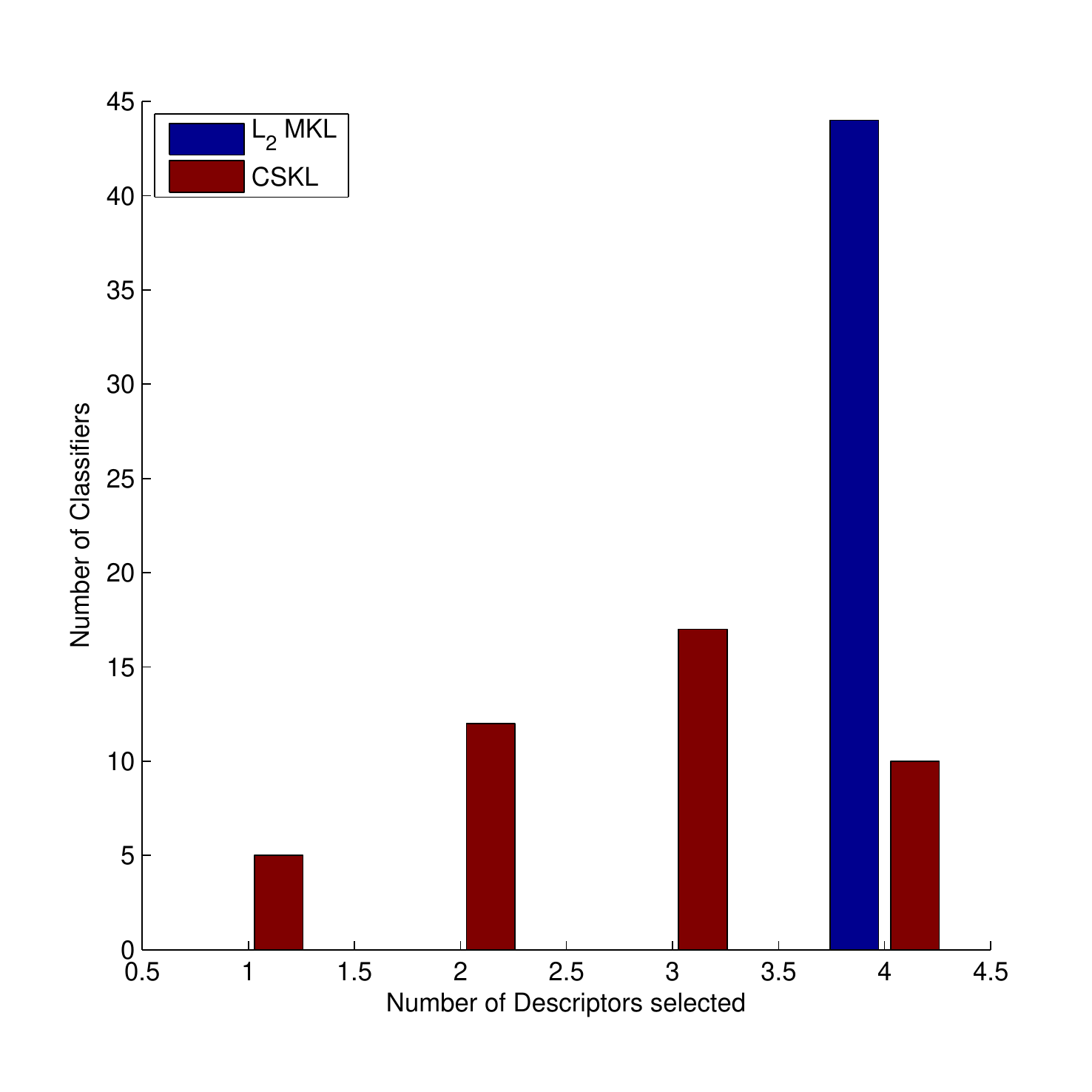}}
\caption{Histogram of number of descriptors selected where {\tt CSKL} performs better than {\tt $L_2$ MKL} in Caltech101 1-vs-Rest binary classification problems} \label{fig:Kloft_vs_BMKL_1vrest_hist_bar}
\end{figure}

\begin{figure}
\centering
\subfloat[Caltech-101]{\includegraphics[width=3in,height=3in]{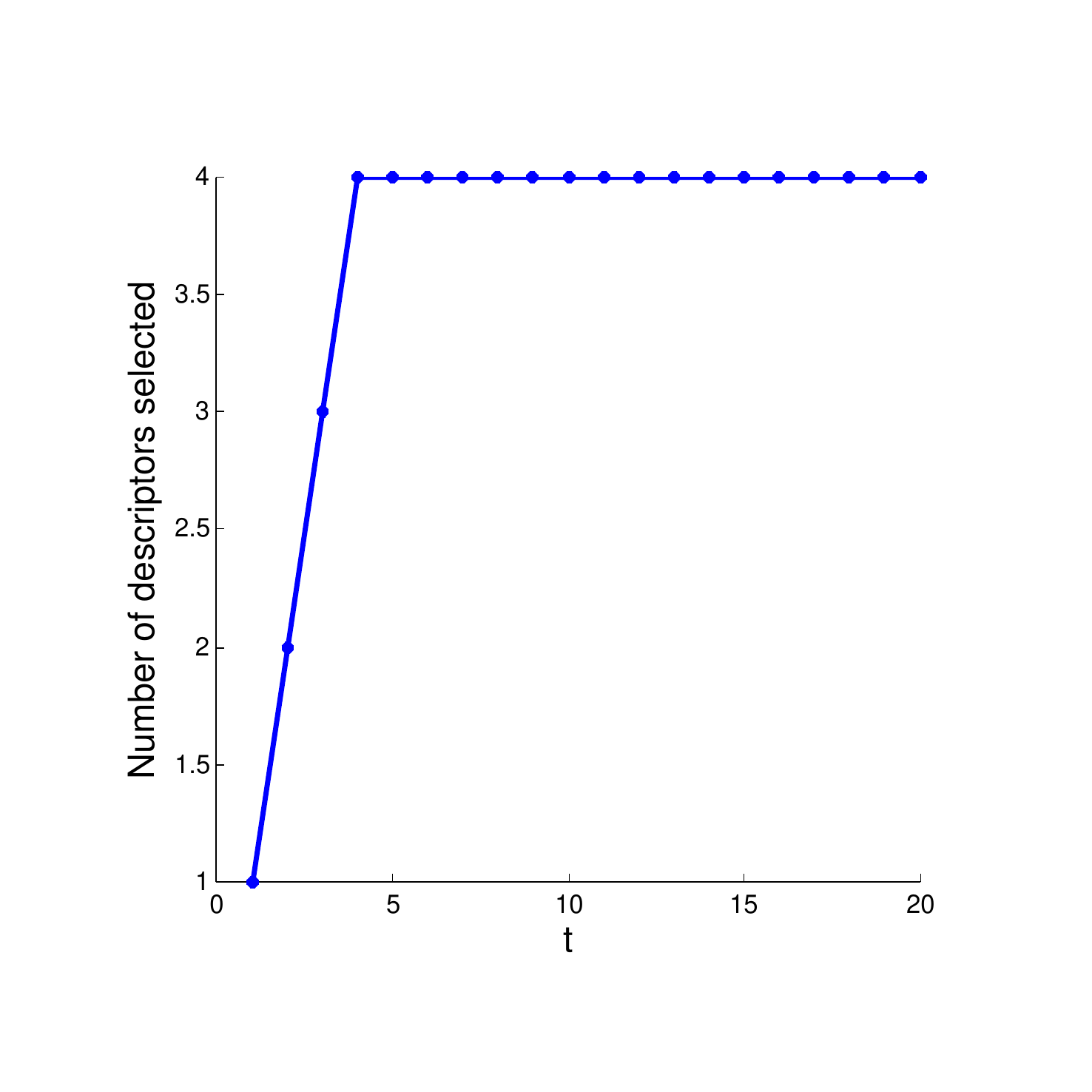}}
\caption{Number of descriptors selected as $t$ is increased} \label{fig:caltech101_desc}
\end{figure}


\section{Conclusion.}
As we have seen, both Sparse and Non-Sparse MKL have their handicaps depending on the classification problem at hand. Niehter of them are always the best. 
Also, in such problems the time taken to calculate the features is one the biggest bottlenecks. For all these reasons, a formulation with strict control of sparsity would be the best solution to have. 
One can then tune the sparsity parameter $t$ and select the best set
of kernels for any particular classification problem. We have described one such formulation in this paper along with the associated solution algorithms. We have also shown the 
superior performance of this formulation with respect to both the Sparse and Non-Sparse formulations of MKL for the application problem of Object Categorization.

\end{document}